\documentclass[]{gewu}

\usepackage{wasysym}
\usepackage[toc,page,header]{appendix}
\usepackage{minitoc}
\usepackage{url}
\usepackage{amssymb}
\usepackage{amsmath}
\usepackage{mathtools}
\usepackage{amsthm}
\usepackage{wrapfig}
\usepackage{enumitem}
\usepackage{arydshln}
\usepackage{gensymb}
\usepackage{makecell}
\usepackage{algorithm}
\usepackage{algorithmic}  
\theoremstyle{plain}
\newtheorem{theorem}{Theorem}

\theoremstyle{definition}

\theoremstyle{remark}

\title{Information-Theoretic Decomposition for \\ Multimodal Interaction Learning}

\author[1,2,3]{Zequn Yang}
\author[1,2,3]{Yake Wei}
\author[4]{Haotian Ni}
\author[5]{Zhihao Xu}
\author[1,2,3,\textnormal{\Letter}]{Di Hu}

\affiliation[1]{Gaoling School of Artificial Intelligence, Renmin University of China, Beijing}
\affiliation[2]{Beijing Key Laboratory of Research on Large Models and Intelligent Governance}
\affiliation[3]{Engineering Research Center of Next-Generation Intelligent Search and Recommendation, MOE}
\affiliation[4]{Beihang University, Beijing}
\affiliation[5]{Gaotu Techedu Inc.}

\contribution[\textnormal{\Letter}]{Corresponding author}

\MyEmail{Zequn Yang at \email{zqyang@ruc.edu.cn}, Di Hu at \email{dihu@ruc.edu.cn}}

\checkdata[Code]{\url{https://github.com/GeWu-Lab/DMIL}}

\abstract{Multimodal learning hinges on capturing redundant, unique, and synergistic information across modalities, which collectively constitute multimodal interactions. A critical yet underexplored challenge is that these implicit interactions vary dynamically across samples. In this work, we present the first systematic, information-theoretic analysis highlighting why learning these dynamic, sample-specific interactions is critical for effective multimodal learning. Our analysis further reveals deficits in conventional paradigms at learning these distinct interaction types: modality ensemble approaches struggle to capture synergy, while joint learning paradigms often under-utilize redundant information. This highlights the need for an approach that can adaptively learn from different interaction types on a per-sample basis.
To this end, we propose \textit{Decomposition-based Multimodal Interaction Learning} (DMIL), a novel paradigm that explicitly models and learns from sample-specific interactions. First, we design a variational decomposition architecture to isolate the constituent interaction components. Second, we employ a new learning strategy that leverages these explicit interaction components in a fine-tuning process to achieve comprehensive interaction learning. Extensive experiments across diverse tasks and architectures demonstrate that DMIL consistently achieves superior performance by adapting to holistic sample-specific interactions. Our framework is flexible and broadly applicable, establishing an interaction-centric paradigm for multimodal learning. The code is available at https://github.com/GeWu-Lab/DMIL.}

\begin{document}
\maketitle
\section{Introduction}



Multimodal interactions describe the diverse pathways through which information is generated and integrated from multiple modalities. For instance, one primary pathway is information consistency, where different modalities provide overlapping or aligned information, such as an image and its caption describing the same concept~\cite{bossard2014food}. Conversely, information can be complementary, where crucial insights emerge from the interplay or discrepancies between modalities. A clear example is sarcasm detection, which often relies on a positive facial expression paired with a negative vocal tone~\cite{attardo2003multimodal}.
To formalize these pathways, Liang et al.~\cite{liang2023quantifying} introduced a framework for \textit{multimodal interaction}, categorizing how multimodal information is composed. This framework identifies information as either \textit{redundant} (shared across modalities), \textit{unique} (specific to one modality), or \textit{synergistic} (emerging only from their joint integration). Understanding these data-intrinsic interactions is crucial, as they define optimal pathways for information processing and inform the design of multimodal models. This is exemplified by approaches promoting cross-modal consistency~\cite{cui2024novel,zhang2024cross} or maximizing information integration~\cite{liang2024factorized, dufumier2024align}. Consequently, a formal understanding and systematic exploration of these interactions are critical for developing advanced and trustworthy multimodal learning systems.

\begin{wrapfigure}{r}{0.5\linewidth}
    \centering
    \includegraphics[width=\linewidth]{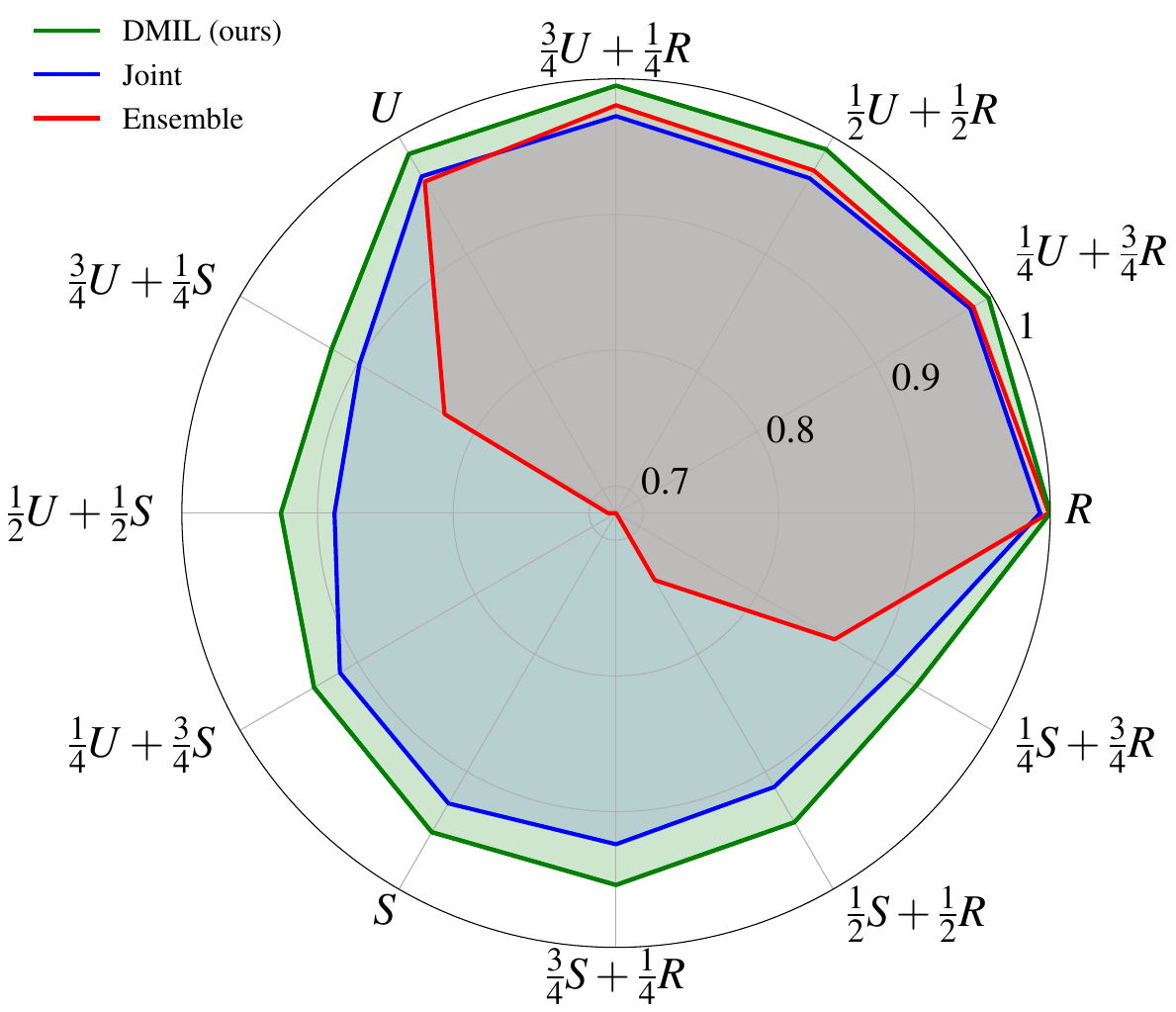}
    \caption{Accuracy comparison of multimodal learning paradigms under varying interaction compositions. The axes denote mixtures of \textbf{U}nique (U), \textbf{S}ynergistic (S), and \textbf{R}edundant (R) samples. For instance, $\frac{1}{4} U + \frac{3}{4} R$ is a 25\% \textbf{U}nique and 75\% \textbf{R}edundant mix.}
    \label{toy_data}
\end{wrapfigure}

Understanding multimodal interactions begins with their definition and quantification. Prior research has leveraged the information decomposition approach to explicitly quantify multimodal interactions~\cite{liang2023quantifying}. This quantification has been instrumental in improving multimodal learning quality, for instance, by guiding model selection~\cite{liang2023multimodal} and partitioning training data~\cite {yang2025Efficient}. Beyond quantification, interactions have also guided specialized training strategies that enhance multimodal learning, allowing samples to fully exploit different interactions~\cite{dufumier2024align, kontras2024multimodal, yang2024Quantifying}.
Despite these advancements, the sample-specific dynamics of multimodal interactions and their impact on model performance remain largely unanalyzed. This gap motivates a fundamental question: \textit{How do these data-intrinsic interactions influence the performance of multimodal learning paradigms?}



To address this question, we introduce an information-theoretic perspective to analyze sample-specific interaction dynamics by modeling their composition. From this perspective, we derive a theoretical bound demonstrating that effective multimodal learning necessitates capturing the full spectrum of interactions: redundancy, uniqueness, and synergy. Guided by this principle, we systematically analyze conventional multimodal paradigms (including joint learning and modality ensembles), revealing their inherent limitations in addressing diverse interaction compositions. 
Experiments using samples with different interaction compositions demonstrate our findings: the performance of modality ensemble degrades significantly under high synergy (the bottom-left region with \textbf{S}ynergy), whereas joint learning struggles to sufficiently capture redundant interactions (e.g., in the $\frac{1}{2}U+\frac{1}{2}R$ case). This analysis not only explains previously observed performance discrepancies~\cite{peng2022balanced,hua2024reconboost} in multimodal learning but also highlights the critical need for more adaptive interaction learning mechanisms.



To address this gap, we introduce \emph{Decomposition-based Multimodal Interaction Learning} (DMIL), a novel paradigm designed to explicitly decompose and enhance the core components of multimodal interaction. 
DMIL decomposes a multimodal representation into its redundant, unique, and synergistic components via an interaction decomposition framework based on variational inference~\cite{alemi2016deep}. Building on this decomposition, we introduce a learning strategy to ensure the information quality of the disentangled components and to enhance the learning of interactions at hand. This principled approach of decomposition and targeted enhancement enables DMIL to dynamically adapt its information processing strategy to the specific interaction of each data sample. As a result, DMIL effectively navigates diverse interaction patterns, leading to superior performance over existing paradigms, as illustrated in ~\autoref{toy_data}. Extensive experiments confirm that DMIL is a powerful and versatile paradigm, consistently outperforming state-of-the-art methods on various datasets. Our contributions are as follows:

\begin{enumerate} 
\item We employ an information-theoretic perspective to analyze how sample-specific interaction dynamics impact multimodal learning, and further reveal the inherent limitations of conventional paradigms.
\item We introduce \textit{Decomposition-based Multimodal Interaction Learning} (DMIL), a novel paradigm that explicitly decomposes multimodal interactions into their constituent components (redundancy, uniqueness, and synergy) and enables holistic interaction learning via finetuning.
\item We conduct extensive experiments across diverse datasets and architectures, demonstrating DMIL's consistent advanced performance and broad applicability.
\end{enumerate}


In summary, we establish the first information-theoretic framework for analyzing the impact of sample-specific interaction dynamics on multimodal learning and propose DMIL, a novel paradigm that explicitly decomposes and enhances diverse interactions, thereby offering valuable insights for advancing multimodal research.

\section{Related Work}
\subsection{Theoretical foundations of multimodal learning}
Despite the notable achievements of multimodal models~\cite{alayrac2022flamingo, openai2023gpt4v}, which achieves remarkable modality cooperation \cite{du2025crab}, fundamental questions remain unresolved in multimodal learning. A primary research direction examines the comparative advantages of multimodal versus unimodal approaches, with theoretical analyses demonstrating benefits in enhanced representation quality \cite{huang2021makes}, reduced learning complexity \cite{lu2023theory}, and improved computational efficiency~\cite{lu2024computational}. Concurrently, another significant research stream investigates intrinsic challenges in multimodal learning systems. This includes studies on modality competition~\cite{huang2022modality} and imbalance~\cite{peng2022balanced}, as well as examinations of data-induced learning biases~\cite{gat2020removing,zhang2024understanding} and contrastive learning dynamics with consistent data~\cite{ren2023importance}. Different from these established perspectives, our work focuses on analyzing how fundamental data interaction, specifically, the mechanisms through which information is generated from multiple modalities, influences multimodal learning performance.

\subsection{Information theory for multimodal learning}

Information theory provides a principled framework for understanding how models learn from multiple data sources~\cite{alemi2016deep, bounoua2024somegai}, guiding effective modeling of inter-modal relationships. Within this framework, some approaches seek to enhance inter-modal consistency, often leveraging the information bottleneck principle~\cite{federici2020learning, cui2024novel}, while others aim to maximize joint information by minimizing inter-modal redundancy~\cite{liang2024factorized}. The pathways through which modal information arises, including redundancy, uniqueness, and synergy, critically influence modeling quality~\cite{miao2026mibench}. Accordingly, research has concentrated on defining and quantifying these multimodal interactions using information decomposition methods~\cite{pakman2021estimating, liang2023quantifying, bertschinger2014quantifying}. Such characterizations not only provide intuitive guidance for model selection~\cite{liang2023multimodal} but have also been shown to improve multimodal learning outcomes~\cite{dufumier2024align, kontras2024multimodal}. Advancing this perspective, our work investigates the relationship between data-specific interaction dynamics and the information acquired by multimodal learning paradigms.
\section{Method}

\subsection{Preliminaries}

We consider predicting a target variable $Y$ from two source modalities, represented by the random variables $X^{(1)}$ and $X^{(2)}$. We use uppercase letters for random variables, lowercase for their realizations, and $I(\cdot;\cdot)$ to denote mutual information. From an information-theoretic perspective, the joint information that $X^{(1)}, X^{(2)}$ provide about $Y$, denoted $I(X^{(1)}, X^{(2)}; Y)$, can be decomposed into distinct components that capture their interaction \cite{liang2023quantifying, bertschinger2014quantifying,williams2010nonnegative}. Quantities of these components, which we denote with a tilde (e.g., $\tilde{R}$), are defined as follows:
\textbf{Redundancy} ($\tilde{R}$): information about $Y$ common to both $X^{(1)}$ and $X^{(2)}$.
\textbf{Uniqueness} ($\tilde{U}^{(1)}, \tilde{U}^{(2)}$): information about $Y$ unique to $X^{(1)}$ or $X^{(2)}$, respectively.
\textbf{Synergy} ($\tilde{S}$): new information about $Y$ that emerges only when considering $X^{(1)}$ and $X^{(2)}$ jointly.
These components quantify how the total multimodal information is structured, satisfying the following relations:
\begin{equation}
    \begin{aligned}
        I(X^{(1)}&, X^{(2)}; Y) = \tilde{R} + \tilde{U}^{(1)}+ \tilde{U}^{(2)}+ \tilde{S}, \\
        I(X^{(1)}; Y) &= \tilde{R} + \tilde{U}^{(1)}, \quad I(X^{(2)}; Y) = \tilde{R} + \tilde{U}^{(2)}.
    \end{aligned}
\label{interaction_relation}    
\end{equation}
In practice, given the complex distributions of raw modalities $X=(X^{(1)}, X^{(2)})$, they are often encoded into more compact representations, collectively denoted as $Z=(Z^{(1)}, Z^{(2)})$ (e.g., via encoders $\phi^{(1)}, \phi^{(2)}$), before being fed into fusion architectures. This representation $Z$ serves as a surrogate for $X$ that fully maintains the information $X$ provides about $Y$.
  
\subsection{Interaction perspective of multimodal learning}

While interaction quantities (e.g., $\tilde{R}, \tilde{U}^{(1)}, \tilde{U}^{(2)}, \tilde{S}$) describe the average amount of information across a dataset, this aggregate view conceals sample-specific variation. In practice, the underlying information structure differs greatly from one sample to the next~\cite{yan2024causality, yang2025Efficient}. For instance, some samples can be fully addressed by information from a single modality, whereas others fundamentally require the synergistic integration of multiple sources. Furthermore, as shown in \autoref{fig_interaction}, completing a task often demands a specific, organic combination of these different interaction types (i.e., redundancy, uniqueness, and synergy). This clearly demonstrates that interaction properties are not uniform; rather, they vary substantially at the individual sample level.

\begin{wrapfigure}{r}{0.5\linewidth}
    \centering
    \includegraphics[width=\linewidth]{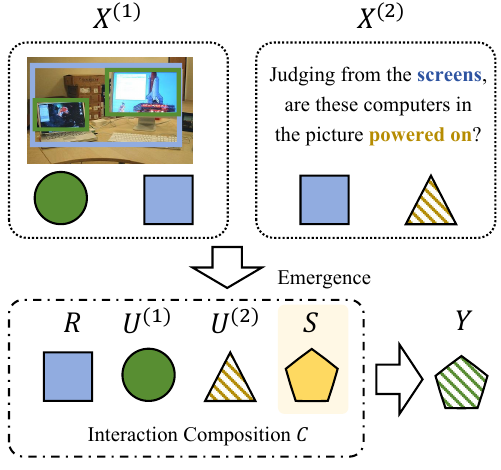}
    \caption{Visualization for Interaction Composition. Each modality contains redundant (in blue) and unique information (\textit{e.g.} green for image modality). Synergy denotes the information emerges beyond the sum of the individual modalities. Interaction composition describes the combination of information to accomplish the final task.}
    \label{fig_interaction}
\end{wrapfigure}
  
This sample-level heterogeneity in interaction presents a significant challenge for multimodal learning, as different models have inherent biases and varying capabilities for processing interactions~\cite{liang2023quantifying}. This makes it difficult to guarantee that they can capture the full spectrum of interaction information, especially when these properties change from sample to sample. Moreover, these interactions are often implicit in the data, and it remains unclear exactly how they impact overall learning performance.


We formalize these sample-specific variations by introducing the \textit{interaction composition}, represented by a random variable $C$ (see \autoref{fig_interaction}). For any given sample, $C$ defines the specific combination of redundant, unique, and synergistic information that modalities $X^{(1)}$ and $X^{(2)}$ provide about the target $Y$.
Consequently, samples with an identical $C$ are assumed to follow similar information pathways, effectively instantiating the macroscopic components from Equation \ref{interaction_relation} at the sample level. This fine-grained characterization enables analyses of model behavior under distinct interaction patterns. We proceed to demonstrate the critical importance of this concept, showing that the ability to learn comprehensively across all interaction compositions is a key determinant of high-quality multimodal learning.

\begin{theorem}
    \label{proposition_31_large} Let $c$ denote the interaction composition, and
    let $I(Z; Y)$ represent the mutual information between $Z$ and $Y$ as
    modeled by a multimodal model. The following inequality holds:
    \begin{equation}
        I(Z;Y) \geq \mathbb{E}_{c}[I(Z;Y|c)] - H(C|Z) + I(Y;C),
        \label{proposition_31}
    \end{equation}
\end{theorem}
The proof is provided in the Appendix A. Theorem \ref{proposition_31_large} establishes a lower bound on the learned multimodal information $I(Z; Y)$, revealing two key factors for high-quality multimodal learning. First, the term $\mathbb{E}_{c}[I(Z;Y|c)]$ represents the model's average performance across all interaction compositions. Maximizing this term requires the model to learn effectively from the full spectrum of interaction dynamics. Second, the bound is tightened by minimizing the uncertainty term $H(C|Z)$, which is achieved when representation $Z$ is highly predictive of the sample's interaction composition $c$. According to the Data Processing Inequality, a representation $Z$ derived from the input $X$ cannot contain more information about the interaction composition $C$ than $X$ itself. Thus, to minimize the uncertainty $H(C|Z)$, the representation $Z$ must optimally preserve the interaction-related information inherent in the original data. In summary, achieving high $I(Z;Y)$ necessitates a model that performs effectively across diverse interaction types and learns representations that explicitly encode the specific interaction composition of each data.

\begin{figure}[t]
    \centering
    \includegraphics[width=\linewidth]{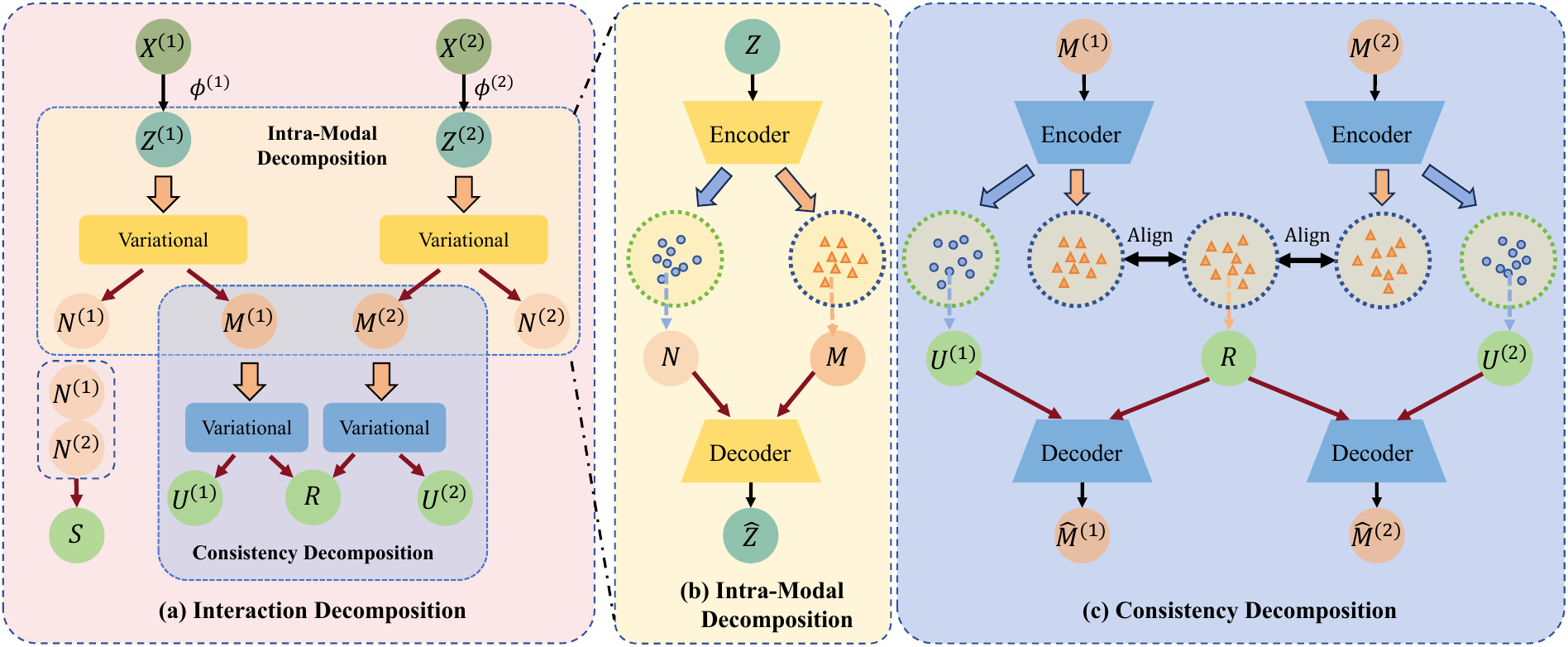}
    \caption{(a) Overall illustration of our proposed \emph{Interaction Decomposition Module}. (b) Intra-Modality Decomposition applies a variation-based
    decomposition to extract the intra-modality information for each modality. (c)
    Consistency Decomposition applies to intra-modality variables and separates consistent
    interactions $R$ from specific ones $U^{(1)}, U^{(2)}$. }
    \label{fig:illu_decom_sim}
\end{figure}

A critical question in multimodal learning is whether existing paradigms can fully capture the rich interaction dynamics inherent in the data. Currently, most paradigms handle these interaction dynamics implicitly. Shaped by distinct inductive biases, they possess varying capabilities for capturing these interactions~\cite{liang2023quantifying}. This reliance on implicit modeling leads to a significant interaction deficit, where models favor certain interaction types over others and fail to leverage the full potential of multimodal data. To illustrate this deficit, we analyze the limitations of two canonical paradigms: joint learning and modality ensemble.
  
\begin{wrapfigure}{r}{0.5\linewidth}
    \centering
    \includegraphics[width=\linewidth]{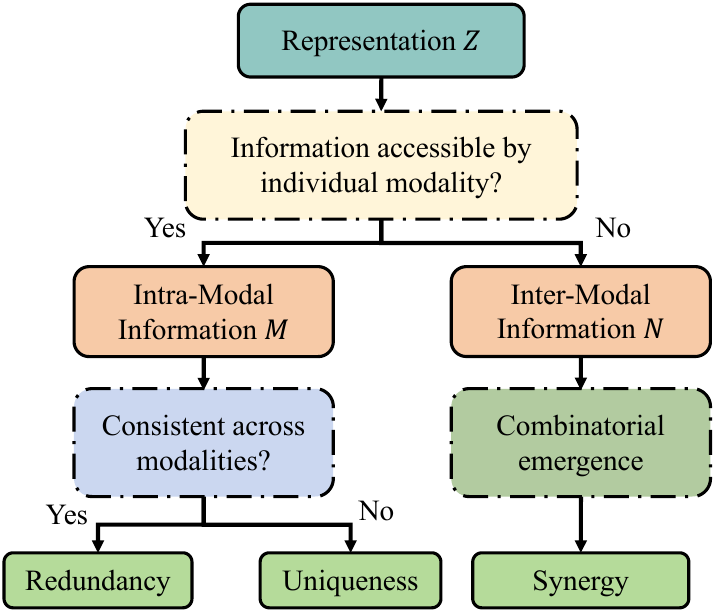}
    \caption{Flowchart for interaction decomposition.}
    \label{fig:proc_for_decomposition}
\end{wrapfigure}

\paragraph{Joint Learning} These methods project modalities into a shared latent space for joint prediction. However, this approach is susceptible to modality competition~\cite{huang2022modality} and imbalance~\cite{peng2022balanced,wu2022characterizing}. In highly redundant scenarios, where multiple modalities provide sufficient task-relevant information, the learning process can be dominated by the most salient or easily-learned modality. This dominance suppresses contributions from other modalities, hindering the model's ability to fully leverage redundant information and degrading overall performance. As illustrated in \autoref{toy_data}, joint learning underperforms other paradigms in the redundancy-rich ($\frac{1}{2}U+\frac{1}{2}R$) setting.

\paragraph{Modality Ensemble} This paradigm involves training unimodal models independently and fusing their predictions at the decision level. While this approach excels at preserving unimodal information and is thus adept at capturing redundancy, it is inherently limited in modeling cross-modal dependencies~\cite{hua2024reconboost}. By design, these models process modalities in isolation, preventing them from learning the synergistic information that emerges only from their joint interaction. Consequently, as shown in \autoref{toy_data}, the performance of modality ensembles degrades sharply in synergy-dominant scenarios, falling significantly behind other paradigms.


In summary, while conventional multimodal learning paradigms possess the capability for interaction modeling, they struggle to learn effectively from diverse interaction compositions. To address this interaction deficit, we propose a novel paradigm designed to first explicitly decompose multimodal information into its constituent interaction components and then enhance them through targeted learning.


\subsection{Learning from multimodal interaction}

We decompose multimodal information into three complementary interaction types, where are redundancy, uniqueness, and synergy, and learn them in a targeted manner. As illustrated in \autoref{fig:proc_for_decomposition}, our framework follows two successive decomposition steps. First, for each modality, we separate its representation into an \emph{intra-modality component}, which is directly predictive of the target, and an \emph{inter-modality component}, which is not independently predictive but serves as the basis for synergistic interactions across modalities. Second, we further decompose the intra-modality components into a shared redundant part and modality-specific unique parts. Based on these decomposed factors, we adopt a three-stage optimization strategy to stabilize training and to ensure that each interaction type is learned appropriately.

Given the unimodal representation $Z^{(m)}$, \emph{intra-modality decomposition} module factorizes it into an intra-modality component $M^{(m)}$ and an inter-modality component $N^{(m)}$:
\begin{equation}
    \begin{aligned}
        \max \quad &I(Z^{(m)}; M^{(m)}, N^{(m)}) + I(M^{(m)}; Y) \\
        - &I(Z^{(m)}; M^{(m)}) - I(Z^{(m)}; N^{(m)}), \quad m \in [2].
    \end{aligned}
    \label{decomp_1}
\end{equation}

This objective drives $M^{(m)}$ to preserve task-relevant information via $I(M^{(m)}; Y)$ while disentangling it from the residual factor $N^{(m)}$. Consequently, $M^{(m)}$ contains directly predictive information, whereas  is insufficient for prediction in isolation but can contribute through cross-modal interaction. We use $N^{(m)}$ to construct the synergy component in later stages. We then perform \emph{consistency decomposition} on $M^{(1)}$ and $M^{(2)}$ to separate shared and modality-specific predictive information:
\begin{equation}
    \begin{aligned}
        \max~ 2 I(M^{(1)}; M^{(2)}; R) - I(U^{(1)}; R) - I(U^{(2)}; R).
    \end{aligned}
    \label{decomp_2}
\end{equation}
Here, $R$ denotes the redundant component shared across modalities, while $U^{(1)}$ and $U^{(2)}$ denote the unique components specific to each modality. Maximizing the tripartite information term encourages $R$ to capture the information consistently present in both modalities, whereas minimizing $I(U^{(m)}; R)$ promotes separation between shared and modality-specific factors. Consequently, $R$ models cross-modal consensus, and $U^{(m)}$ preserves the predictive information that cannot be explained by the shared component.

To optimize these decomposed representations, we use a three-stage training strategy. In Stage 1, we train the encoder and the intra-modality decomposition module to learn $M^{(m)}$ and $N^{(m)}$. In Stage 2, we freeze the Stage-1 modules to stabilize the learned representation space, and optimize the consistency decomposition module to extract $R$ and $U^{(m)}$ from $M^{(m)}$. Concurrently, complementing the decomposition of $M^{(m)}$, we construct the synergy component $S$ by fusing the residual factors $N^{(1)}$ and $N^{(2)}$ through a multi-level fusion mechanism. 
In Stage 3, all parameters are jointly fine-tuned. Specifically, each decomposed component $c \in \{R, U^{(1)}, U^{(2)}, S\}$ is first projected into the output space via a linear mapping to obtain its initial prediction $\hat{y}_c$. We then employ a gating network to dynamically predict component-specific weights $g_c$, aggregating these predictions to produce the final output:
\begin{equation}
    \hat{Y} = \sum_{c \in \{R, U^{(1)}, U^{(2)}, S\}} g_c \hat{y}_c.
\end{equation}
This final stage refines the full model end-to-end while preserving the learned decomposed interaction structure.

To empirically optimize the mutual information objectives, we define the stage-wise losses as follows. Specifically, the target loss $\mathcal{L}_{\text{t}}$ (acting as a proxy for maximizing $I(\cdot; Y)$) is applied to $M^{(m)}$, $R$, $U^{(m)}$, $S$, and the final prediction $\hat Y$ to ensure task relevance. The variational loss $\mathcal{L}_{\text{Var}}$ is utilized to approximate the minimization of mutual information terms (e.g., $I(Z^{(m)}; N^{(m)})$ and $I(U^{(m)}; R)$), encouraging disentanglement between latent factors. The alignment loss $\mathcal{L}_{\text{Align}}$ is imposed on $R$ to enforce cross-modal consistency. Accordingly, the stage-wise objectives are defined as:

\begingroup\small
\begin{equation}
\mathcal{L}_1 = \sum_{m \in [2]} \left[ \mathcal{L}_{\text{t}}(M^{(m)}) + \mathcal{L}_{\text{Var}}(M^{(m)}, N^{(m)}) \right],
\end{equation}
\begin{equation}
\mathcal{L}_2 = \sum_{c \in \{R, U, S\}} \mathcal{L}_{\text{t}}(c)
+ \sum_{m \in [2]} \mathcal{L}_{\text{Var}}(U^{(m)}, R)
+ \alpha \mathcal{L}_{\text{Align}}(R),
\end{equation}
\begin{equation}
\mathcal{L}_3 = \mathcal{L}_{\text{t}}(\hat{Y}) + \beta \mathcal{L}_1 + \gamma \mathcal{L}_2.
\end{equation}
\endgroup
In this way, synergy is explicitly learned in Stage 2 and further refined in Stage 3, while redundancy and uniqueness are learned under dedicated structural constraints. 
Additional details, including the derivation of the objectives (Appendix A.3) and the synergy component derivation (Appendix B.2), are provided in the appendices.
 
\begin{table}[t]
    \centering
    \small
    \caption{Performance comparison (Accuracy and mAP) on multiple multimodal datasets. All results are reported as mean $\pm$ standard deviation. The best results are \textbf{highlighted}, and ``--'' denotes non-applicability.}
    \setlength{\tabcolsep}{3pt}
    \renewcommand{\arraystretch}{1.13}
    \newcommand{\std}[1]{{\tiny$\pm$#1}}
    \resizebox{\linewidth}{!}{
    \begin{tabular}{l|cc|cc|cc|cc|cc}
    \toprule
    \multirow{2}{*}{Method}
        & \multicolumn{2}{c|}{CREMA-D}
        & \multicolumn{2}{c|}{KS}
        & \multicolumn{2}{c|}{UCF101}
        & \multicolumn{2}{c|}{KS (ViT)}
        & \multicolumn{2}{c}{CMU-MOSEI} \\
    \cmidrule(lr){2-3} \cmidrule(lr){4-5} \cmidrule(lr){6-7} \cmidrule(lr){8-9} \cmidrule(lr){10-11}
    & \textbf{ACC} & \textbf{mAP}
    & \textbf{ACC} & \textbf{mAP}
    & \textbf{ACC} & \textbf{mAP}
    & \textbf{ACC} & \textbf{mAP}
    & \textbf{ACC} & \textbf{mAP} \\
    \midrule
    Joint                           & 72.55\std{0.3}            & 81.12\std{0.4}      & 85.07\std{0.3}            & 90.93\std{0.1}      & 79.33\std{0.5}            & 86.66\std{0.3}      & 67.81\std{0.5}            & 73.11\std{0.8}      & 79.42\std{0.3}            & 80.76\std{0.8}      \\
    Ensemble                        & 74.87\std{0.7}            & 82.35\std{0.4}      & 85.86\std{0.4}            & 91.90\std{0.4}      & 83.63\std{0.3}            & 89.77\std{0.4}      & 71.80\std{0.5}            & 79.03\std{0.6}      & 79.18\std{0.4}            & 80.79\std{0.1}      \\
    \midrule
    OGM~\cite{peng2022balanced}     & 72.95\std{0.3}            & 80.67\std{0.4}      & 85.17\std{0.4}            & 90.78\std{0.4}      & 79.94\std{0.4}            & 86.32\std{0.4}      & 68.02\std{0.1}            & 73.42\std{0.5}      & 79.57\std{0.3}            & 80.07\std{0.7}      \\
    PMR~\cite{fan2023pmr}           & 72.45\std{0.6}            & 81.14\std{0.5}      & 86.03\std{0.3}            & 91.51\std{0.2}      & 79.18\std{0.5}            & 85.53\std{0.5}      & 67.99\std{0.3}            & 73.36\std{1.1}      & 80.34\std{0.1}            & 81.74\std{0.7}      \\
    AGM~\cite{kontras2024improving} & 73.51\std{0.5}            & 82.17\std{0.4}      & 85.41\std{0.5}            & 91.40\std{0.4}      & 79.52\std{0.4}            & 86.26\std{0.4}      & 68.51\std{0.1}            & 75.43\std{0.4}      & 79.40\std{0.3}            & 80.54\std{0.5}      \\
    \midrule
    MMTM~\cite{vaezi20mmtm}         & 73.44\std{0.5}            & 82.08\std{0.4}      & 84.10\std{0.4}            & 90.43\std{0.5}      & 79.40\std{0.3}            & 85.57\std{0.6}      & --                         & --                   & --                         & --                   \\
    MMIB~\cite{mai2022multimodal}   & 74.59\std{0.5}            & 83.23\std{0.4}      & 85.86\std{0.4}            & 91.75\std{0.2}      & 81.52\std{0.4}            & 88.09\std{0.3}      & 67.57\std{0.6}            & 73.52\std{0.3}      & 79.73\std{0.1}            & 81.65\std{0.3}      \\
    QMF~\cite{zhang2023provable}    & 73.32\std{0.4}            & 81.16\std{0.5}      & 85.82\std{0.4}            & 91.15\std{0.5}      & 80.63\std{0.4}            & 85.67\std{0.5}      & 68.10\std{0.9}            & 72.90\std{1.5}      & 79.44\std{0.2}            & 80.91\std{0.7}      \\
    FCL~\cite{liang2024factorized}  & 74.73\std{0.6}            & 83.08\std{0.6}      & 85.81\std{0.4}            & 91.98\std{0.4}      & 81.47\std{0.4}            & 88.31\std{0.3}      & 68.90\std{0.6}            & 75.38\std{1.0}      & 79.81\std{0.2}            & 81.68\std{0.4}      \\
    MLB~\cite{kontras2024improving} & 75.94\std{0.4}            & 83.83\std{0.2}      & 85.52\std{0.3}            & 91.74\std{0.2}      & 83.25\std{0.2}            & 89.53\std{0.1}      & 71.22\std{0.8}            & 77.29\std{0.7}      & 79.56\std{0.2}            & 80.71\std{0.6}      \\
    MMML~\cite{wu2024multimodal}    & 74.46\std{0.4}            & 83.03\std{0.5}      & 85.92\std{0.4}            & 91.90\std{0.3}      & 82.23\std{0.5}            & 88.50\std{0.5}      & 69.14\std{0.2}            & 75.82\std{0.3}      & 79.87\std{0.2}            & 81.04\std{0.2}      \\
    MCR~\cite{kontras2024multimodal} & 76.34\std{0.5}           & 84.59\std{0.6}      & 86.41\std{0.6}            & 92.07\std{0.4}      & 82.81\std{0.5}            & 88.93\std{0.3}      & 69.51\std{0.6}            & 76.33\std{0.8}      & 80.77\std{0.3}            & 81.80\std{0.5}      \\
    I2M2~\cite{madaan2024jointly}   & 74.13\std{0.2}            & 82.90\std{0.4}      & 85.07\std{0.3}            & 91.11\std{0.3}      & 80.02\std{0.7}            & 86.24\std{0.6}      & 68.93\std{0.9}            & 73.55\std{0.7}      & 79.85\std{0.2}            & 81.25\std{0.3}      \\
    \midrule
    \textbf{DMIL}                   & \textbf{77.02\std{0.3}}  & \textbf{85.13\std{0.4}} & \textbf{86.72\std{0.3}}  & \textbf{92.15\std{0.2}} & \textbf{85.01\std{0.3}}  & \textbf{90.53\std{0.4}} & \textbf{74.20\std{0.4}}  & \textbf{81.83\std{0.5}} & \textbf{81.12\std{0.4}}  & \textbf{82.17\std{0.8}} \\
    \bottomrule
    \end{tabular}
    }
    \label{main_exp_cnn}
\end{table}

\section{Experiment}

This section presents the primary experimental results. A detailed account of the experimental setup and comprehensive analyses (e.g. backbone comparisons, data dynamics, and interaction weight analysis) are provided in Appendix B.

\subsection{Experiment setting}
\textbf{Datasets}: 
\textbf{CREMA-D} \cite{cao2014crema}: An audio-visual emotion recognition dataset containing 7,442 clips across six emotion categories.
\textbf{Kinetic-Sounds} (KS)~\cite{arandjelovic2017look}: A multimodal action recognition dataset featuring 19,000 clips covering 31 human action classes, using both audio and visual data.
\textbf{UCF101} \cite{soomro2012ucf101}: An action recognition dataset with 13,320 videos across 101 classes, utilizing RGB and optical flow modalities. (Additional synthetic experiments are detailed in the Appendix). \textbf{CMU-MOSEI} \cite{zadeh2018multimodal}: A large-scale sentiment analysis dataset with 23,453 segments. We focus on the visual and textual modalities for a two-way classification task (positive, negative). Details about real-world and synthetic datasets are provided in the Appendix B. 
  
\textbf{Experimental detail}:
We report the mean and standard deviation over 5 runs for all results. Inputs from two modalities are first processed by their respective encoders to extract features, and are then fed into the fusion module. In the CNN-based experiments, we use ResNet18~\cite{he2016deep} as the backbone. In the Transformer-based experiments, for the CMU-MOSEI dataset, we follow the preprocessing steps and apply Transformers to encode different modalities in line with~\cite{liang2021multibench}. For the KS and VGGSound datasets, we also utilize Vision Transformer (ViT) to encode image and audio. 

\textbf{Method realization}:
In our proposed DMIL, interaction feature are projected into the output space via linear mapping. We employ a gating network to dynamically aggregate these projected components based on their contributions.  Furthermore, we introduce a novel fusion paradigm for the $S$ynergy component, specifically designed to capture complex synergistic interactions (e.g., XOR) by transforming them into a linearly separable representation. Further experimental details and explanation are provided in the Appendix B.


\subsection{Comparison on real-world datasets}
\label{real-world-dataset}
To comprehensively evaluate our proposed \textit{Decomposition-based Multimodal Interaction Learning} (DMIL) framework, we conduct extensive comparisons with multimodal approaches across diverse real-world datasets. We categorize the baseline methods into three groups for systematic analysis: (1) \textbf{Conventional Baselines}, including conventional multimodal paradigms joint training and modality ensemble; (2) \textbf{Regularization-based Methods} that improve multimodal learning by constraining unimodal representations (OGM~\cite{peng2022balanced}, PMR~\cite{fan2023pmr}, and AGM~\cite{kontras2024improving}); and (3) \textbf{Interaction-focused Methods} explicitly designed to model cross-modal relationships (MMTM~\cite{vaezi20mmtm}, MMIB~\cite{mai2022multimodal}, QMF~\cite{zhang2023provable}, FCL~\cite{liang2024factorized}, MLB~\cite{kontras2024improving}, MMML~\cite{wu2024multimodal}, MCR ~\cite{kontras2024multimodal}, and I2M2~\cite{madaan2024jointly}). The comprehensive results are presented in Table~\ref{main_exp_cnn}. `-' in the table indicates cases where a method is not applicable to a specific dataset.

\begin{wrapfigure}{r}{0.5\linewidth}
    \centering
    \includegraphics[width=\linewidth]{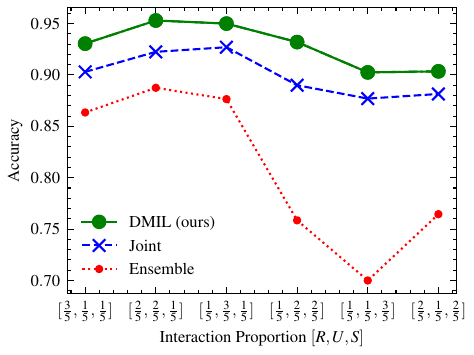}
    \caption{Performance comparison with varying compositions of Redundancy, Uniqueness, and Synergy interactions.}
    \label{toy_data_3}
\end{wrapfigure}
  
From these empirical results, we draw several key observations. First, joint modeling demonstrates superior performance over ensemble learning only in specific scenarios, depending on both the dataset characteristics and model architecture. Second, while both regularization methods and architecture-based interaction learning approaches generally outperform joint modeling in enhancing multimodal learning, their effectiveness varies across different experimental settings. For instance, methods like MLB excel on certain datasets, such as UCF101, but show limited gains on others, highlighting the lack of a universally optimal fusion strategy among prior works. Notably, on the Kinetic-Sounds dataset, the ensemble method surpasses most competing approaches, likely due to the dataset's inherent propensity for exhibiting more readily learnable interactions. Third, our DMIL framework consistently achieves state-of-the-art performance across both backbone architectures. This superior performance stems from our method's ability to effectively decompose and strengthen the learning from diverse multimodal interactions. Unlike methods that implicitly handle these interactions, DMIL's explicit disentanglement allows for a more targeted and robust fusion, leading to consistent improvements. Finally, the effectiveness of our approach is further validated by its maintained performance advantage when applied to Transformer-based architectures and text-related tasks, underscoring the method's generalizability.

\subsection{Validation for interaction learning}
  \begin{wraptable}{r}{0.5\linewidth}
    \centering
    \small
    \vspace{-1em}
    \caption{Synergy validation on VQAv2 and CLEVR datasets.}
    \renewcommand{\arraystretch}{0.9}
    \setlength{\tabcolsep}{2.5pt}
    \begin{tabular}{lccccccc}
        \toprule
        \multirow{2}{*}{Method} & \multicolumn{3}{c}{VQAv2} & & \multicolumn{3}{c}{CLEVR} \\
        \cmidrule(lr){2-4} \cmidrule(lr){6-8}
         & ACC & mAP & Syn. & & ACC & mAP & Syn. \\
        \midrule
        Ensemble & 58.03 & 59.58 & 0.00 & & 58.35 & 62.35 & 0.00 \\
        Joint    & 67.69 & 74.03 & 9.48 & & 62.65 & 69.90 & 7.53 \\
        \textbf{DMIL} & \textbf{70.08} & \textbf{77.54} & \textbf{11.44} & & \textbf{63.08} & \textbf{71.18} & \textbf{9.69} \\
        \bottomrule
    \end{tabular}
    \label{synergy_dataset}
\end{wraptable}
\paragraph{Interaction learning performance.}
\label{radar_text}
Our analysis and method focus on learning from different types of interactions, which are difficult to distinguish and measure directly from the real-world data. To address this, we construct synthetic datasets in which the intrinsic interactions within the data can be manually configured. In this setup, each sample contains a specific type of data—Redundancy ($R$), Uniqueness ($U$), or Synergy ($S$)—based on how they relate to the target. For redundancy data, both modalities map to the target, whereas for uniqueness data, only one modality is predictive. In synergy data, each modality contributes partial but insufficient information (similar to the XOR condition). We evaluate DMIL against standard paradigms—Joint Learning, Modality Ensemble —on datasets with varying mixtures of R, U, and S samples (\autoref{toy_data} and \ref{toy_data_3}). The results confirm our analysis that joint training shows significant improvement in synergy-related data, while unimodal ensemble performs better with redundancy types of interactions. In contrast, DMIL achieves substantial performance gains across all scenarios, demonstrating the advantage of its targeted approach to learning from diverse interactions.




\paragraph{Synergy Experiment.} Synergistic interaction is a fundamental yet challenging aspect of multimodal learning, as it represents essential information that only emerges through the joint reasoning of multiple modalities. To explicitly verify whether our method can capture these complex cross-modal dynamics, we evaluate it on two synergy-dominated datasets~\cite{liang2023quantifying}, VQAv2 and CLEVR. Following~\cite{kontras2024multimodal}, we employ the \textit{Synergy} metric, which quantifies the proportion of samples correctly predicted by the multimodal model while simultaneously misclassified by both unimodal counterparts:
{\small
\begin{equation}
    Synergy = \frac{1}{N} \sum_{i=1}^{N} \mathbb{I}(\hat y_i = y_i \wedge \hat y_i^{(1)} \neq y_i \wedge \hat y_i^{(2)} \neq y_i),
\end{equation}
}
where $\mathbb{I}(\cdot)$ denotes the indicator function, $y_i$ is the ground-truth label, and $\hat y_i$ and $\hat y_i^{(m)}$ denote the predictions of the multimodal model and the unimodal models, respectively. As reported in \autoref{synergy_dataset}, our proposed DMIL consistently outperforms standard baselines and obtains the highest Synergy scores. Compared with Ensemble, which only combines unimodal outputs, and Joint, which performs standard multimodal training, the consistent gains of DMIL indicate that our method is more effective as it can better capturing synergy interaction. These results suggest that the performance improvement of DMIL comes from its stronger ability to learn from complex modality interactions.

\begin{wrapfigure}{r}{0.5\linewidth}
    \centering
    \vspace{-1em}
    \includegraphics[width=\linewidth]{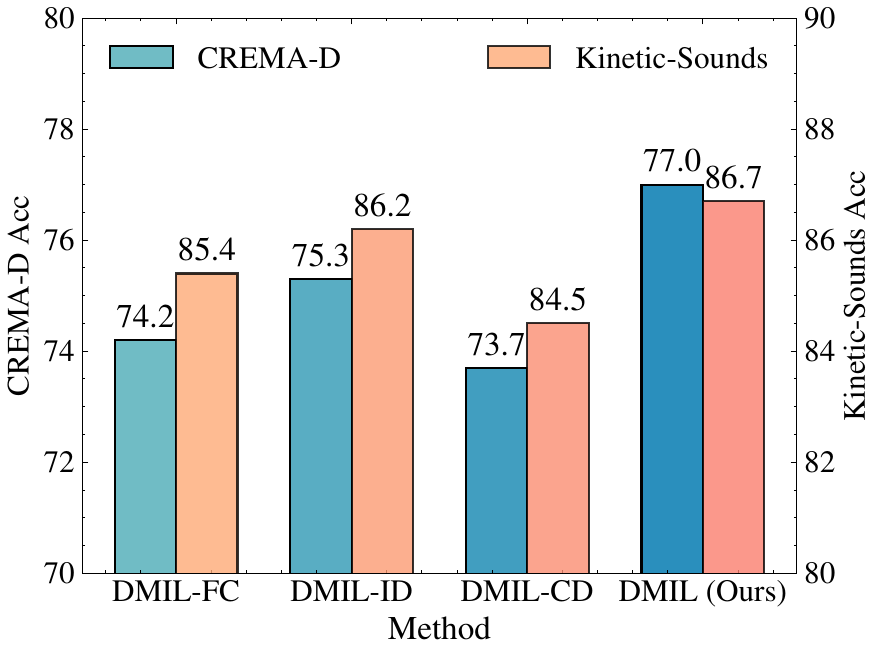}
    \caption{Ablation studies of DMIL paradigm on CREMA-D and Kinetic-Sounds (KS) datasets.}
    \label{ablation}
    \vspace{-1em}
\end{wrapfigure}

\paragraph{Ablation study}
In this section, we conduct an ablation study to validate the core components of our interaction decomposition framework. The study addresses two key questions: \textit{1) Is the variational approach crucial for decomposition? 2) Are both the Intra-modality (ID) and Consistency Decomposition (CD) modules necessary for optimal performance?} To answer these, we evaluate three ablated variants against the full DMIL model, with results shown in \autoref{ablation}. First, \textit{DMIL-FC} replaces the variational layers with fully connected layers to test the variational approach. Second, to isolate the contribution of each decomposition module, we create \textit{DMIL-ID} (which retains only the ID module) and \textit{DMIL-CD} (which retains only the CD module).

The results in \autoref{ablation} confirm our design choices, as the full DMIL model substantially outperforms all ablated versions. The performance drop from DMIL to \textit{DMIL-FC} underscores the importance of the variational method for effectively decoupling interactions. Furthermore, the lower performance of \textit{DMIL-ID} and \textit{DMIL-CD} compared to the full model demonstrates that both decomposition modules are essential. These findings collectively validate that the variational approach and the dual-module (ID and CD) architecture are integral to DMIL's effectiveness.

\begin{figure}[t]
    \centering
    \includegraphics[width=0.85\linewidth]{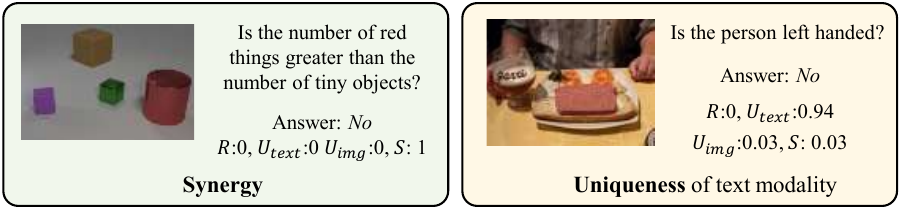}
    \caption{Case study: Synergy vs. Textual Uniqueness.}
    \label{case_study}
\end{figure}

\subsection{Interaction under different scenarios}
\paragraph{Interaction in three modalities. } Extending multimodal methods beyond two modalities, a common requirement for real-world scenarios, presents significant challenges. We examine this extension from both analytical and experimental perspectives. Analytically, interactions among more than two modalities introduce substantial complexity. For instance, defining mutual information across three or more variables to characterize concepts like redundancy, uniqueness, and synergy remains a significant theoretical challenge~\cite{liang2023quantifying}.

Experimentally, however, our proposed Decomposition-based Multimodal Interaction Learning (DMIL) approach can be adapted for scenarios involving three modalities. Specifically, our paradigm can accommodate three modalities by incorporating Intra-modality Decomposition (DMIL-ID), as illustrated in \autoref{ablation}. We conducted empirical evaluations on two three-modality datasets: CMU-MOSEI (Visual, Audio, and Text) and UCF101 (RGB, Optical Flow, and Frame Difference). As shown in \autoref{three_mod}, the experimental results demonstrate the improvements from our DMIL method, highlighting its flexibility and effectiveness.

\paragraph{Case Study.} \autoref{case_study} further provides sample-level evidence for the effectiveness of our decomposition framework. In the left example, neither the image nor the text alone is sufficient to answer the question, and our method assigns a Synergy score of 1, which is fully consistent with the definition of synergistic interaction. In the right example, the prediction is primarily supported by textual prior knowledge (right-handedness is relatively common), and our method correspondingly assigns a larger contribution to text uniqueness. These two cases show that our framework can distinguish synergy-dominant samples from unimodal-dominant ones, which supports both its dataset-level performance gains and its interpretable, fine-grained decomposition of sample-level multimodal interactions.




\begin{table}[t]
    \centering
    \small
    \begin{minipage}[t]{0.48\linewidth}
        \centering
        \caption{Validation on trimodal setting.}
        \label{three_mod}
        \setlength{\tabcolsep}{4pt}
        \begin{tabular}{l|cc|cc}
            \toprule
            \multirow{2}{*}{Method} & \multicolumn{2}{c|}{MOSEI} & \multicolumn{2}{c}{UCF} \\
            \cmidrule(lr){2-3} \cmidrule(lr){4-5}
            & ACC & mAP & ACC & mAP \\
            \midrule
            Joint    & 80.31 & 81.43 & 79.25 & 86.57 \\
            Ensemble & 79.78 & 81.29 & 84.76 & 90.07 \\
            \textbf{DMIL} & \textbf{82.13} & \textbf{82.95} & \textbf{85.81} & \textbf{91.24} \\
            \bottomrule
        \end{tabular}
    \end{minipage}
    \hfill
    \begin{minipage}[t]{0.48\linewidth}
        \centering
        \caption{OOD validation on KS and CREMA-D.}
        \label{ood_table}
        \setlength{\tabcolsep}{4pt}
        \begin{tabular}{l|cc|cc}
            \toprule
            \multirow{2}{*}{Method} & \multicolumn{2}{c|}{KS} & \multicolumn{2}{c}{CREMA-D} \\
            \cmidrule(lr){2-3} \cmidrule(lr){4-5}
            & ID & OOD & ID & OOD \\
            \midrule
            Joint    & 85.54 & 52.95 & 90.19 & 67.11 \\
            Ensemble & 90.96 & 53.65 & 91.01 & 71.88 \\
            \textbf{DMIL} & \textbf{91.71} & \textbf{56.04} & \textbf{91.28} & \textbf{72.53} \\
            \bottomrule
        \end{tabular}
    \end{minipage}
\end{table}
\paragraph{OOD generalization.} 

To evaluate the generalization and scalability of DMIL, we conducted an Out-of-Distribution experiment. We partitioned the classes from the KS and CREMA-D datasets into two disjoint sets: In-Distribution (ID) classes used for training and Out-of-Distribution (OOD) classes held out for testing. The model was trained exclusively on the ID set, and we then evaluated its ability to distinguish the representations of the unseen OOD classes. As shown in \autoref{ood_table}, DMIL performs well on both ID and OOD sets. This suggests that by explicitly learning multimodal interactions, our method avoids overfitting to training-specific correlations and achieves strong generalization, thus highlighting the scalability of our approach.

\section{Conclusion}
We introduce an information-theoretic framework that highlights the importance of learning from different interaction compositions. Additionally, we analyze that previous multimodal learning paradigms are not able to learn from various interaction compositions. Based on this analysis, we propose a \textit{Decomposition-based Multimodal Interaction Learning} (DMIL) paradigm that effectively distinguishes and strengthens interactions within the data, and we design a three-stage learning process to achieve improved performance.

\textbf{Future Work}: We identify two primary directions for future work. First, investigating how to effectively capture and interpret interaction mechanisms within MLLMs across a broader and more complex range of task scenarios. Second, scaling the modeling of interactions to settings that involve a larger number of modalities. This direction will require developing efficient methods to manage the combinatorial complexity introduced by higher-order interactions.

\section*{Acknowledgments}

This work is supported by Beijing Natural Science Foundation (4262050), the Public Computing Cloud, Renmin University of China and the fund for building world-class universities (disciplines) of Renmin University of China.

\clearpage

\bibliographystyle{IEEEtran}
\bibliography{paper}

\clearpage

\beginappendix

In the Supplementary Material, we first present the proofs and detailed analysis for the information-based interaction bound and the decomposition architecture in \autoref{proof_analysis}. Specifically, we provide the theoretical proof for Theorem 1 to clarify the significance of interaction \autoref{proof_for_pro31}. And we provide the variational explanation for the interaction decomposition architecture \autoref{explanation_for_decom}. Next, we describe the experimental setup and model architecture in greater detail \autoref{experimental_setting}. Finally, we conduct expanded evaluations across diverse domains with varying modalities and backbones to further validate the applicability of our DMIL method \autoref{expanded_evaluation}.
\section{Proof and Analysis}
\label{proof_analysis}
\subsection{Proof for Theorem 1}
\label{proof_for_pro31}

\begin{theorem}
    Let \(C\) be the interaction composition random variable with realization \(c\), \(Z\) be the learned representation from a multimodal input \(X\), and \(Y\) be the target variable. The mutual information \(I(Z;Y)\) is lower-bounded as follows:
    \begin{equation}
        I(Z;Y) \geq \mathbb{E}_{c}[I(Z;Y|c)] - H(C|Z) + I(Y;C),
        \label{prop_31_appendix}
    \end{equation}
    where \(H(C|Z)\) is the conditional entropy of \(C\) given \(Z\).
    \end{theorem}

    \begin{proof}
    The proof aims to establish a lower bound for \(I(Z;Y)\) by analyzing the difference between the total mutual information and the expected conditional mutual information, a quantity we denote as \(\Delta\). This term, \(\Delta = I(Z;Y) - \mathbb{E}_{c}[I(Z;Y|c)]\), quantifies the information synergy; that is, how much information about \(Y\) is revealed in \(Z\) due to knowledge of the interaction context \(C\).

    Our derivation rests on two key assumptions regarding the data generation process:
    \begin{enumerate}
        \item \textbf{Deterministic Encoding}: The representation \(Z\) is generated by a deterministic encoder \(\phi\) from the multimodal input \(X\), i.e., \(Z = \phi(X)\). This implies that given \(X\), there is no uncertainty about \(Z\), so the conditional entropy \(H(Z|X) = 0\). This establishes the Markov chain \(Y \to X \to Z\).
        \item \textbf{Deterministic Interaction}: The interaction variable \(C\) is a deterministic function of the input \(X\) and the target \(Y\), i.e., \(C = f(X,Y)\). This means the conditional entropy \(H(C|X,Y) = 0\). The variable \(C\) is designed to capture specific interaction patterns between modalities in \(X\) that are relevant for predicting \(Y\).
    \end{enumerate}

    We begin by expressing \(\Delta\) using a standard information-theoretic identity. We have the following identity:
    \begin{equation}    
        \begin{split}
            \Delta &= \mathbb{E}_{z,y} \log \frac{p(z,y)}{p(z)p(y)} -\mathbb{E}_{z,y,c} \log 
            \frac{p(z,y|c)}{p(z|c) p(y|c)} \\
            &= \mathbb{E}_{z,y,c} \log \frac{p(y|c)p(z|c)p(z,y)}{p(y)p(z)p(z,y|c)} \\
            &= I(Y; C) + I(Z; C) - I(Z, Y; C).
        \end{split}
    \end{equation}
    This identity follows from the chain rule of mutual information, as \(I(Z,Y;C) = I(Z;C) + I(Y;C|Z)\) and \(I(Z;Y) - I(Z;Y|C) = I(Y;C) - I(Y;C|Z)\).

    Next, we apply the Data Processing Inequality (DPI). From our first assumption (\(Z=\phi(X)\)), the pair \((Z,Y)\) is a function of the pair \((X,Y)\). This implies the Markov chain \((X,Y) \to (Z,Y)\). The DPI states that post-processing cannot increase information, so for any variable \(C\), we have:
    \begin{equation}
        I(Z,Y;C) \leq I(X,Y;C).
    \end{equation}
    Substituting this into our expression for \(\Delta\) yields the inequality:
    \begin{equation}
        \Delta \geq I(Y;C) + I(Z;C) - I(X,Y;C).
    \end{equation}

    Now, we use our second assumption, that \(C\) is a deterministic function of \(X\) and \(Y\). This means the conditional entropy \(H(C|X,Y) = 0\). By the definition of mutual information:
    \begin{equation}
        I(X,Y;C) = H(C) - H(C|X,Y) = H(C) - 0 = H(C).
    \end{equation}
    Replacing \(I(X,Y;C)\) with \(H(C)\) in our inequality for \(\Delta\), we get:
    \begin{equation}
        \Delta \geq I(Y;C) + I(Z;C) - H(C).
    \end{equation}

    To simplify further, we use the definition of mutual information for \(I(Z;C)\):
    \begin{equation}
        I(Z;C) = H(C) - H(C|Z).
    \end{equation}
    Rearranging this gives \(I(Z;C) - H(C) = -H(C|Z)\). Substituting this into the inequality for \(\Delta\):
    \begin{equation}
        \Delta \geq I(Y;C) - H(C|Z).
    \end{equation}

    Finally, by substituting back the definition of \(\Delta\), we arrive at the desired lower bound:
    \begin{equation}
        I(Z;Y) - \mathbb{E}_{c}[I(Z;Y|c)] \geq I(Y;C) - H(C|Z),
    \end{equation}
    which can be rearranged to conclude the proof:
    \begin{equation}
        I(Z;Y) \geq \mathbb{E}_{c}[I(Z;Y|c)] - H(C|Z) + I(Y;C).
    \end{equation}
    \end{proof}

\subsection{Analysis for Theorem 1}

Theorem 1 establishes a lower bound on the learned multimodal information $I(Z; Y)$. As derived in its proof, this bound is given by:
\begin{equation*}
    I(Z;Y) \geq \mathbb{E}_{c}[I(Z;Y|c)] - H(C|Z) + I(Y;C).
\end{equation*}
To maximize this lower bound and thus enhance the overall multimodal information captured by the model, we must consider the roles of its constituent terms. The term $I(Y; C)$ is determined by the intrinsic data distribution and represents the total information about the target $Y$ contained in the interaction composition $C$. The other two terms, $\mathbb{E}_{c}[I(Z;Y|c)]$ and $H(C|Z)$, are directly influenced by the learned representation $Z$.

Firstly, $\mathbb{E}_{c}[I(Z;Y|c)]$ represents the expected conditional mutual information between the representation $Z$ and the target $Y$, given the interaction composition $c$. This term quantifies how well the model captures task-relevant information for specific interaction patterns. Maximizing this term implies that the model effectively learns information across diverse interaction dynamics, which is crucial for robust multimodal learning.

Secondly, $H(C|Z)$ is the conditional entropy of the interaction variable $C$ given the learned representation $Z$. This term measures the uncertainty remaining about the interaction composition $C$ after observing $Z$. A lower $H(C|Z)$ indicates that $Z$ is a better predictor of $C$, meaning the representation effectively captures the underlying interaction patterns. According to the Data Processing Inequality, $H(C|Z) \geq H(C|X)$, which means $H(C|X)$ serves as a fundamental lower bound for $H(C|Z)$, representing the inherent uncertainty of interaction given the raw data. To maximize the overall multimodal information $I(Z;Y)$, it is crucial for the representation $Z$ to minimize $H(C|Z)$, thereby learning to predict the intrinsic interaction compositions $C$ as accurately as possible.

In summary, the proof of Theorem 1 underscores that the ability of multimodal representations to effectively capture information across diverse interaction types (maximizing $\mathbb{E}_{c}[I(Z;Y|c)]$) and to accurately predict these intrinsic interaction patterns (minimizing $H(C|Z)$) are critical factors for achieving comprehensive and high-quality multimodal information learning.

\subsection{Explanation of decomposition}
\label{explanation_for_decom}

\subsubsection{Intra-Modal Decomposition}

Here, we provide a detailed explanation of how our proposed decomposition framework operates. The core objective is to decompose a given representation $Z$ into two latent components, $N$ and $M$. The decomposition aims to satisfy two primary conditions: (1) the components $N$ and $M$ should be statistically independent, minimizing their mutual information $I(N; M)$, and (2) they should collectively preserve the information content of the original representation $Z$. To achieve this, we employ a variational approach.

We begin by deriving the Evidence Lower Bound (ELBO) for the marginal log-likelihood of $Z$, $\log p(z)$. We assume that $z$ is generated from latent variables $n$ and $m$ drawn from independent priors, i.e., $p(n, m) = p(n)p(m)$. We introduce a factorized variational posterior $q(n, m|z) = q(n|z)q(m|z)$ to approximate the true posterior $p(n, m|z)$. By applying Jensen's inequality, we obtain the ELBO:
\begin{equation}
\begin{aligned}
    \log p(z) &= \log \int p(z|n,m)p(n)p(m) \, dv \, dt \\
    &\geq \mathbb{E}_{q(n|z), q(m|z)} \left[ \log \left( \frac{p(z|n,m)p(n)p(m)}{q(n|z)q(m|z)} \right) \right] \\
    &= \mathbb{E}_{q(n|z), q(m|z)} [\log p(z|n,m)] \\
    &\quad + \mathbb{E}_{q(n|z)} \left[ \log \frac{p(n)}{q(n|z)} \right] + \mathbb{E}_{q(m|z)} \left[ \log \frac{p(m)}{q(m|z)} \right] \\
    &= \mathbb{E}_{q(n|z), q(m|z)} [\log p(z|n,m)] \\
    &\quad - KL(q(n|z) || p(n)) - KL(q(m|z) || p(m)).
\end{aligned}
\label{ELBo}
\end{equation}
Maximizing this lower bound involves optimizing three terms. The first term, $\mathbb{E}_{q(n|z), q(m|z)} [\log p(z|n,m)]$, is the reconstruction term, which encourages the latent components $N$ and $M$ to accurately reconstruct the original representation $Z$. The other two terms are regularization terms that minimize the Kullback-Leibler (KL) divergence between the variational posteriors ($q(n|z)$, $q(m|z)$) and their respective priors ($p(n)$, $p(m)$). Typically, the priors are chosen to be standard Gaussian distributions, $\mathcal{N}(0, I)$, to encourage well-structured latent spaces.

To understand why this decomposition architecture achieves disentanglement, we can analyze its connection to information theory. The goal is to decompose a representation $Z$ into two statistically independent components, $N$ and $M$. This can be framed as an optimization problem aimed at minimizing their mutual information, $I(N; M)$, while ensuring that $N$ and $M$ collectively retain the information from $Z$.

From an information-theoretic standpoint, minimizing $I(N; M)$ is equivalent to maximizing the negative interaction information, $I(Z; M, N) - I(Z; M) - I(Z; N)$. This objective encourages the joint representation $(N, M)$ to be predictive of $Z$ (maximizing $I(Z; M, N)$) while penalizing the information that $N$ and $M$ individually share with $Z$ (minimizing $I(Z; M)$ and $I(Z; N)$). This process effectively isolates the unique contributions of $N$ and $M$ and minimizes their redundant overlap.

This equivalence relies on the assumption that $N$ and $M$ are conditionally independent given $Z$, i.e., $p(n,m|z) = p(n|z)p(m|z)$. This is a natural assumption for a model designed to decompose $Z$ into distinct factors and is explicitly enforced in our variational framework by the factorized posterior $q(n,m|z) = q(n|z)q(m|z)$. The derivation is as follows:
\begin{equation}
\begin{aligned}
    &I(Z; M, N) - I(Z; M) - I(Z; N) \\
    &= \mathbb{E}_{p(z,v,m)} \left[ \log \frac{p(z,m,n) p(z) p(m) p(n)}{p(m,n) p(z,m) p(z,v)} \right] \\
    &= \mathbb{E}_{p(z,v,m)} \left[ \log \frac{p(m,n|z) p(m) p(n)}{p(m,n) p(m|z) p(n|z)} \right] \\
    &= \mathbb{E}_{p(z,v,m)} \left[ \log \frac{p(m|z)p(n|z) p(m) p(n)}{p(m,n) p(m|z) p(n|z)} \right] \\
    &= \mathbb{E}_{p(m,n)} \left[ \log \frac{p(m)p(n)}{p(m,n)} \right] = -I(M; N).
\end{aligned}
\label{eq:interaction_info_mi}
\end{equation}
Thus, maximizing the interaction information objective is equivalent to minimizing the mutual information $I(M;N)$. We now show that maximizing the ELBO from Equation \ref{ELBo} aligns with this information-theoretic objective. The ELBO can be expressed as:
\begin{equation}
\begin{aligned}
    \mathcal{L}_{\text{ELBO}} &= \underbrace{\mathbb{E}_{q(n|z), q(m|z)} [\log p(z|n,m)]}_{\text{Reconstruction}} \\
    &\quad - \underbrace{KL(q(n|z) || p(n))}_{\text{Regularization for } N} \\
    &\quad - \underbrace{KL(q(m|z) || p(m))}_{\text{Regularization for } M}.
\end{aligned}
\end{equation}
Each term in the ELBO corresponds to a term in our information-theoretic objective:
\begin{enumerate}
    \item \textbf{Reconstruction Term}: Maximizing the reconstruction term, $\mathbb{E}[\log p(z|n,m)]$, is equivalent to minimizing the conditional entropy $H(Z|V,M)$. Since $I(Z; V, M) = H(Z) - H(Z|V,M)$, this term effectively maximizes the joint mutual information $I(Z; V, M)$, ensuring that the latent components collectively preserve information about $Z$.
    \item \textbf{Regularization Terms}: The KL divergence terms serve as variational upper bounds on the mutual information between the representation and the latent components. Specifically, we have:
    \begin{equation}
    \begin{aligned}
        I(Z;M) &= \mathbb{E}_{p(z,m)}\left[\log \frac{p(m|z)}{p(m)}\right]\leq \mathbb{E}_{p(z)}[KL(q(m|z) || p(m))], \\
        I(Z;N) &= \mathbb{E}_{p(z,v)}\left[\log \frac{p(n|z)}{p(n)}\right] \leq \mathbb{E}_{p(z)}[KL(q(n|z) || p(n))].
    \end{aligned}
    \label{KL_appendix}
    \end{equation}
    Maximizing the ELBO involves minimizing these KL terms, which in turn minimizes the upper bounds on $I(Z;M)$ and $I(Z;N)$.
\end{enumerate}
Therefore, maximizing the ELBO in Equation \ref{ELBo} implicitly encourages maximizing $I(Z; V, M)$ while minimizing $I(Z; M)$ and $I(Z; N)$. This aligns directly with the objective of maximizing $I(Z; M, N) - I(Z; M) - I(Z; N)$, which, as shown in Equation \ref{eq:interaction_info_mi}, is equivalent to minimizing $I(M;N)$. This connection demonstrates that our variational decomposition framework is principled and effectively promotes the learning of disentangled latent representations.

\subsubsection{Consistency Decomposition}

After obtaining the intra-modality feature $M^{(m)}$, we propose a consistency decomposition to separate it into a modality-specific vector $U^{(m)}$ and a consistency/shared vector $R$. The objective is to learn representations that capture shared information in $R$ while isolating unique, modality-specific information in $U^{(m)}$. This is achieved by maximizing the following objective function:
\begin{equation}
\label{decomp_2_appendix}
\max ~ 2 I(M^{(1)}; M^{(2)}; R) - I(U^{(1)}; R) - I(U^{(2)}; R).
\end{equation}
This objective encourages $R$ to capture information common to both $M^{(1)}$ and $M^{(2)}$ (interaction information \\ $I(M^{(1)}; M^{(2)}; R)$) while being independent of the unique components $U^{(1)}$ and $U^{(2)}$.

To better understand the optimization process, we can decompose this objective. By focusing on the components related to modality $M^{(1)}$, the objective can be rewritten into three interpretable terms. The derivation is as follows, assuming a symmetric treatment for $M^{(2)}$:
\begin{equation}
\begin{aligned}
    &I(M^{(1)}; M^{(2)}; R) - I(U^{(1)}; R) \\
    &= I(M^{(1)}; R) + I(M^{(2)}; R) \\  &\quad\quad- I(M^{(1)}, M^{(2)}; R) - I(U^{(1)}; R) \\
    &= \underbrace{I(M^{(1)}; U^{(1)}, R)}_{\text{Reconstruction}} - \underbrace{I(M^{(1)}; U^{(1)})}_{\text{Compactness}} - \underbrace{I(M^{(2)}; R | M^{(1)})}_{\text{Redundancy}}.
\end{aligned}
\label{decomp_2_appendix_2}
\end{equation}
The last equation is similar to Equation \ref{eq:interaction_info_mi}.
This decomposition reveals three key optimization goals:
\begin{enumerate}
    \item \textbf{Maximizing Reconstruction}: The term $I(M^{(1)}; U^{(1)}, R)$ corresponds to a reconstruction objective. Maximizing it is equivalent to minimizing the conditional entropy $H(M^{(1)}|U^{(1)}, R)$, ensuring that the original feature $M^{(1)}$ can be accurately reconstructed from its specific component $U^{(1)}$ and the shared component $R$.

    \item \textbf{Maximizing Compactness}: The term $-I(M^{(1)}; U^{(1)})$ encourages the specific representation $U^{(1)}$ to be a compact, minimal representation of the information in $M^{(1)}$, following the information bottleneck principle. This is analogous to the KL divergence regularization term in Equation \ref{KL_appendix}.

    \item \textbf{Minimizing Redundancy}: The term $-I(M^{(2)}; R | M^{(1)})$ aims to minimize the conditional mutual information between $M^{(2)}$ and $R$ given $M^{(1)}$. This encourages $R$ to only contain information that is shared between $M^{(1)}$ and $M^{(2)}$, effectively isolating the redundant (shared) information from the unique aspects of each modality.
\end{enumerate}

The third term, the conditional mutual information $I(M^{(2)}; R | M^{(1)})$, is often intractable to compute directly because it requires evaluating $q(r|m^{(1)}) = \int q(r|m^{(1)}, m^{(2)})p(m^{(2)}|m^{(1)})dm^{(2)}$.
\begin{equation}
\begin{aligned}
  &I(M^{(2)}; R | M^{(1)}) = \mathbb{E}_{p(m^{(1)}, m^{(2)}) q(r|m^{(1)}, m^{(2)})} \left[ \log \frac{q(r | m^{(1)}, m^{(2)})}{q(r|m^{(1)})} \right].
\end{aligned}
\label{decomp_2_appendix_3}
\end{equation}
To address this, we introduce a variational distribution $v_\phi(r|m^{(1)})$ to approximate the true posterior $q(r|m^{(1)})$. This allows us to derive a tractable variational upper bound on this term:
\begin{equation}
\begin{aligned}
    ~& \mathbb{E}_{p(m^{(1)}, m^{(2)}) q(r|m^{(1)}, m^{(2)})} \left[ \log \frac{q(r | m^{(1)}, m^{(2)}) v_\phi(r|m^{(1)})}{v_\phi(r|m^{(1)}) q(r|m^{(1)})} \right] \\
    &= \mathbb{E}_{p(m^{(1)}, m^{(2)})} \left[ \text{KL}(q(r|m^{(1)}, m^{(2)}) || v_\phi(r|m^{(1)})) \right] \\
    &~- \mathbb{E}_{p(m^{(1)})} \left[ \text{KL}(q(r|m^{(1)}) || v_\phi(r|m^{(1)})) \right] \\
    &\leq \mathbb{E}_{p(m^{(1)}, m^{(2)})} \left[ \text{KL}(q(r|m^{(1)}, m^{(2)}) || v_\phi(r|m^{(1)})) \right].
\end{aligned}
\label{decomp_2_appendix_4}
\end{equation}
The inequality holds because the KL divergence is non-negative. By minimizing this upper bound, we effectively minimize $I(M^{(2)}; R | M^{(1)})$. This strategy allows for the practical decomposition of features into their unique and redundant components.
In summary, the complete objective function $\mathcal{L}_{\text{c}}$ is given by:
\begin{equation}
\begin{aligned}
    \mathcal{L}_{\text{c}} & = \underbrace{\mathbb{E}_{q(r, u^{(1)}|m^{(1)})p(m^{(2)}|m^{(1)})} [\log p(m^{(1)}|r, u^{(1)})]}_{\text{Reconstruction}} \\
    &\quad+ \underbrace{\mathbb{E}_{q(r, u^{(2)}|m^{(2)})p(m^{(1)}|m^{(2)})} [\log p(m^{(2)}|r, u^{(2)})]}_{\text{Reconstruction}} \\
    &\quad - \underbrace{KL((q(r|m^{(1)}) || v(r)) + KL(q(r|m^{(2)}) || v(r)))}_{\text{Regularization for } R} \\
    &\quad - \underbrace{KL(q(u^{(1)}|m^{(1)}) || p(u^{(1)}))}_{\text{Regularization for } U^{(1)}} \\
    &\quad - \underbrace{KL(q(u^{(2)}|m^{(2)}) || p(u^{(2)}))}_{\text{Regularization for } U^{(2)}}.
\end{aligned}
\label{decomp_2_appendix_5}
\end{equation}

The final loss function is a composition of the Evidence Lower Bound (ELBO), denoted as $f(z,y) = \mathcal{L}_{\text{ELBO}}$, and the consistency loss, $g(z,y) = \mathcal{L}_{\text{c}}$.

\section{Experiment}
\label{experimental_appendix_overall}
\subsection{Experimental setting}
\label{experimental_setting}

\paragraph{Real-world Dataset}

\textbf{BRCA}~\cite{wang2020moronet}: A dataset for breast invasive carcinoma PAM50 subtype classification, including 875 samples across 5 classes. We utilize mRNA expression and DNA methylation modalities.\\
\textbf{ROSMAP}~\cite{de2018multi}: A dataset targeted at Alzheimer's Disease diagnosis, containing 351 samples spanning 2 classes. Consistent with BRCA, we utilize mRNA expression and DNA methylation modalities.\\
\textbf{CREMA-D}~\cite{cao2014crema}: An emotion recognition dataset consisting of 7,442 video clips. It covers six emotional states (anger, happiness, sadness, neutrality, disgust, and fear) using Audio and Visual modalities.\\
\textbf{Kinetic-Sounds (KS)}~\cite{arandjelovic2017look}: A multimodal action recognition dataset containing 19,000 ten-second clips. It includes 31 human action classes selected from the Kinetics dataset, utilizing Audio and Visual modalities.\\
\textbf{CMU-MOSEI}~\cite{zadeh2018multimodal}: A large-scale sentiment analysis dataset including 23,453 video segments with Audio, Visual, and Text modalities. In this work, we focus specifically on the Visual and Text modalities for a challenging three-way sentiment classification task (positive, negative, neutral).\\
\textbf{UCF101}~\cite{soomro2012ucf101}: An action recognition benchmark containing 13,320 videos across 101 human action classes. We employ RGB and Optical Flow modalities.\\
\textbf{UR-FUNNY}~\cite{hasan2019ur}: A large-scale dataset for humor detection, including over 16,000 samples from TED talks. It integrates Text, Visual, and Acoustic modalities to capture diverse speakers and topics.\\
\textbf{VGGSound}~\cite{chen2020vggsound}: A large-scale audio-visual dataset consisting of short clips from over 200,000 videos, designed to capture sound events in diverse acoustic environments.

\paragraph{Training Details}
All experiments used a batch size of 64. CNN-based models were optimized using SGD (momentum 0.9, weight decay 1e-4) with dataset-specific learning rates: 1e-2 (KS/UCF) and 4e-3 (CREMA-D). For the Transformer-based MOSEI experiments, we used Adam with a learning rate of 1e-4. Training for our proposed method proceeded in three stages: (1) training the intra-modal decomposition for 10 epochs; (2) freezing this module while training subsequent consistency decomposition for 5–10 epochs; and (3) unfreezing the decomposition module to jointly fine-tune the full model with a learning rate between 1e-4 and 1e-5.

\begin{wraptable}{r}{0.55\linewidth}
\centering
\small
\vspace{-1em}
\caption{Performance on diverse tasks with various modality combinations. Best results are \textbf{bolded}.}
\label{extensive_data}
\setlength{\tabcolsep}{2.5pt}
\begin{tabular}{l|cc|cc|cc}
\toprule
\multirow{2}{*}{\textbf{Method}} & \multicolumn{2}{c|}{\textbf{UR-FUNNY}} & \multicolumn{2}{c|}{\textbf{ROSMAP}} & \multicolumn{2}{c}{\textbf{BRCA}} \\
\cmidrule(lr){2-3} \cmidrule(lr){4-5} \cmidrule(lr){6-7}
 & ACC & mAP & ACC & mAP & ACC & mAP \\
\midrule
Joint      & 63.71 & 67.99 & 82.08 & 85.57 & 88.97 & 86.41 \\
Ensemble   & 62.57 & 65.92 & 83.02 & 86.32 & 89.35 & 86.62 \\
DMIL       & \textbf{65.50} & \textbf{69.15} & \textbf{85.85} & \textbf{88.84} & \textbf{89.73} & \textbf{87.33} \\
\bottomrule
\end{tabular}
\vspace{-1em}
\end{wraptable}
\paragraph{Data Preprocessing} For datasets containing videos, we extract frames at 1 fps. In the KS dataset, we uniformly sample 3 frames from each 10-frame clip as visual inputs. For CREMA-D, 1 frame is extracted from each video. In UCF101, we select 2 RGB frames and 5 optical flow frames per video. For VGGSound, 3 frames are used as visual inputs. Additionally, we conduct experiments using more frames for both the KS and CREMA-D datasets; see Section \ref{more_temporal} for details.

  \subsection{Model architecture}
\label{sec_architecture}

\begin{wrapfigure}{r}{0.5\linewidth}
  \centering
  \vspace{-1em}
  \includegraphics[width=\linewidth]{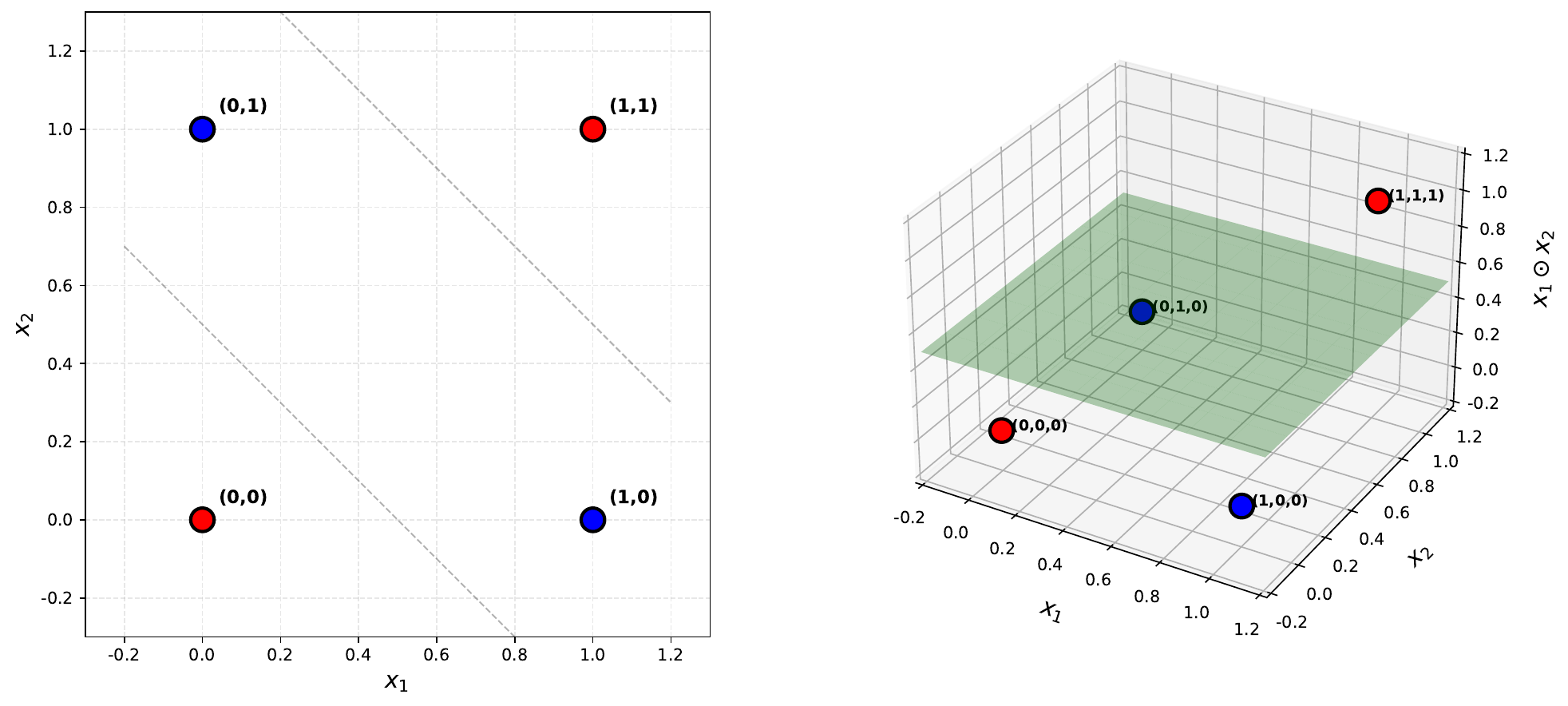}
  \caption{Illustration of the XOR data, which is a simple example of synergistic interaction. Left: The features $(x_1, x_2)$ are not linearly separable. Right: By adding the interaction term $(x_1 \cdot x_2)$, the data becomes linearly separable in 3D.}
  \label{xor_explanation}
\end{wrapfigure}
  
\paragraph{Architecture Description.}In this section, we detail the architecture of our Decomposition-based Multimodal Interaction Learning (DMIL) model, as illustrated in Figure 3. The DMIL framework incorporates two distinct decomposition modules. First, raw inputs from $m$ modalities, denoted as $X^{(m)}$, are mapped into feature representations $Z^{(m)}$ via modality-specific encoders $\phi^{(m)}$. These features are subsequently decomposed to capture inter-modal interactions. Each decomposition module is built upon a Variational Autoencoder (VAE) framework, wherein the encoders—composed of Multi-Layer Perceptrons (MLPs)—predict the mean and variance of the latent distributions. Correspondingly, the decoders employ multi-layer networks designed to minimize information loss during reconstruction. To extract the redundant component, we enforce feature alignment across modalities by minimizing the Kullback-Leibler (KL) divergence between their corresponding distributions.

\paragraph{Synergistic Interaction.}The synergistic component is designed to capture information that is not present in any single modality but emerges solely through their integration. Therefore, designing an effective fusion mechanism is crucial. While simple concatenation is a common approach, it is often insufficient for modeling complex interactions. For instance, in a typical synergistic scenario such as the XOR problem, data points are not linearly separable in the original feature space. This non-linearity hinders the model's ability to learn synergistic interactions effectively. As shown in \autoref{xor_explanation} a, linear separation is unachievable with simple inputs. To address this, we introduce an element-wise product term. The formulation is defined as: $S= f_{\text{linear}}(n^{(1)}, n^{(2)}, n^{(1)} \odot n^{(2)})$, where $\odot$ denotes element-wise multiplication. The introduction of this multiplicative term simplifies the interaction modeling: it renders the XOR data linearly separable, as demonstrated in \autoref{xor_explanation} b. Furthermore, this product term explicitly captures high-order interactions, which are essential for defining synergy.

\begin{table}[t]
\centering
\small
\renewcommand{\arraystretch}{1.1}
\caption{Performance comparison across different backbone architectures (CNN and Vision Transformer) on KS and VGG datasets.}
\label{backbone_change}
\setlength{\tabcolsep}{5pt}
\begin{tabular}{lcccccccc}
\toprule
\multirow{2}{*}{\textbf{Method}} 
    & \multicolumn{2}{c}{\textbf{KS (CNN)}} 
    & \multicolumn{2}{c}{\textbf{VGG (CNN)}} 
    & \multicolumn{2}{c}{\textbf{KS (ViT)}} 
    & \multicolumn{2}{c}{\textbf{VGG (ViT)}} \\
\cmidrule(l){2-3} \cmidrule(l){4-5} \cmidrule(l){6-7} \cmidrule(l){8-9}
 & \textbf{ACC} & \textbf{mAP} 
 & \textbf{ACC} & \textbf{mAP} 
 & \textbf{ACC} & \textbf{mAP}
 & \textbf{ACC} & \textbf{mAP} \\
\midrule
Joint     & 85.07          & 90.93          & 56.79          & 59.21          & 67.81          & 73.11          & 45.09          & 44.84          \\
Ensemble  & 85.86          & 91.90          & 57.46          & 59.85          & 71.80          & 79.03          & 48.16          & 50.03          \\
\textbf{DMIL}   & \textbf{86.72} & \textbf{92.15} & \textbf{58.05} & \textbf{60.66} & \textbf{74.20} & \textbf{81.50} & \textbf{52.28} & \textbf{53.91} \\
\bottomrule
\end{tabular}
    \vspace{-1em}
\end{table}

\subsection{Expanded evaluation on diverse domains}
\subsubsection{Generalization across modalities and tasks} 
To comprehensively validate the robustness of our method, we introduced three additional datasets: UR-FUNNY, ROSMAP, and BRCA. These were selected to cover distinct tasks and modality combinations beyond standard audio-visual benchmarks. UR-FUNNY focuses on humor detection using Visual (V) and Text (T) modalities, which are known to contain strong synergistic information~\cite{liang2023quantifying}. ROSMAP and BRCA are biological datasets involving mRNA and DNA methylation (METH) modalities, presenting a different challenge in feature interaction.
  
\label{expanded_evaluation}
  \begin{wraptable}{r}{0.5\linewidth}
\centering
\small
\caption{Performance with richer temporal dynamics on CREMA-D (5 frames) and KS (8 frames).}
\label{multiple-frame}
\setlength{\tabcolsep}{3pt}
\begin{tabular}{lcccc}
\toprule
\multirow{2}{*}{\textbf{Method}} & \multicolumn{2}{c}{\textbf{CREMA-D (5F)}} & \multicolumn{2}{c}{\textbf{KS (8F)}} \\
\cmidrule(lr){2-3}\cmidrule(lr){4-5}
 & ACC & mAP & ACC & mAP \\
\midrule
Joint     & 79.44 & 87.69 & 85.39 & 91.59 \\
Ensemble  & 81.32 & 90.46 & 86.34 & 92.76 \\
\textbf{DMIL} & \textbf{83.16} & \textbf{91.83} & \textbf{87.44} & \textbf{93.15} \\
\bottomrule
\end{tabular}
  \vspace{-1em}
\end{wraptable}
  
The results are presented in Table \ref{extensive_data}. On the UR-FUNNY dataset, the Joint model outperforms the Ensemble baseline. This supports the idea that humor detection relies heavily on the synergy between vision and text, which Ensemble methods—designed to isolate modalities—often fail to capture. Conversely, on the biological datasets (ROSMAP and BRCA), the Ensemble model performs slightly better than the Joint model. Despite these variations in baseline performance, our DMIL method consistently achieves the highest accuracy and mAP across all datasets. This demonstrates that DMIL can adaptively learn interactions, whether the task requires capturing strong synergy (as in humor) or handling more independent modalities (as in biological analysis).

\subsubsection{Impact of richer temporal dynamics}
\label{more_temporal}
Previous studies highlight the importance of temporal dynamics in multimodal learning~\cite{bernin2018automatic}. Therefore, we investigated how increased temporal resolution affects model performance. We conducted experiments on the CREMA-D and KS datasets using inputs with richer temporal information: 5 frames for CREMA-D and 8 frames for KS. The results are presented in Table \ref{multiple-frame}. Compared to the baseline results in the main manuscript, CREMA-D shows significant improvement with added temporal information, whereas the gain on KS is more moderate. Notably, on CREMA-D, the ensemble model benefits more from the increased frame count than the joint model. This suggests that richer temporal dynamics strengthen the individual unimodal features, which the ensemble method captures more effectively. Crucially, our DMIL method continues to outperform these baselines, demonstrating its effectiveness with varying temporal dynamics.

\subsubsection{Various backbone}

Distinct backbone architectures process data differently, which affects how they capture and learn interactions. Following the initial experiments in Table 1, we conducted further validation on the KS and VGGSound datasets to examine the impact of architecture choice. Specifically, we introduced the widely used Vision Transformer (ViT) as the backbone for both modality encoders to provide a comparison with CNN-based models.
  
\begin{wraptable}{r}{0.45\linewidth}
\centering
\small
\caption{Weights of different interaction types learned on the CREMA-D and KS datasets.}
\label{tab:interaction_weights}
\setlength\tabcolsep{4pt}
\begin{tabular}{lcc}
\toprule
\textbf{Interaction Type} & \textbf{CREMA-D} & \textbf{KS} \\
\midrule
Redundancy    & 0.7471 & 0.5732 \\
Unique-Visual & 0.0650 & 0.2093 \\
Unique-Audio  & 0.1329 & 0.2096 \\
Synergy       & 0.0549 & 0.0078 \\
\bottomrule
\end{tabular}
\end{wraptable}
  
The results, presented in Table \ref{backbone_change}, indicate that the choice of backbone significantly influences baseline performance. We observed that ViT demonstrates weaker modeling capabilities for audio-visual understanding compared to CNN-based models. Additionally, the joint model marginally underperforms the ensemble model, particularly when using ViT. This is likely because audio-visual tasks are complex to learn, and ensemble methods are often more effective at preserving distinct uni-modal information. In contrast, our DMIL method consistently delivers superior performance across different backbones, demonstrating that it not only outperforms other approaches but also effectively learns interactions regardless of the architecture.

\subsubsection{Interaction Weight}

Our method employs an adaptive weighting mechanism to assign importance to different interaction components. Table \ref{tab:interaction_weights} presents the average weights learned across samples for the CREMA-D and KS datasets. We observe that Redundancy is the predominant component in both datasets. Notably, on CREMA-D, the weight for Audio Uniqueness significantly surpasses that of Visual Uniqueness. This suggests that the audio modality possesses stronger discriminative power than the visual modality for this task. In contrast, KS shows a balanced distribution between audio and visual unique weights ($0.2096$ vs. $0.2093$). These results demonstrate that our method effectively captures the intrinsic properties and primary information carriers of each dataset.

\paragraph{Interaction learning demonstration.}
While the previous section confirms DMIL's overall performance, it is crucial to verify that our proposed training strategy effectively learns the intended interaction components. To this end, we construct synthetic datasets from bit-wise features governed by logical relationships with the labels, including mixtures such as $1/2$ AND and $1/2$ XOR, $1/2$ OR and $1/2$ XOR, and $1/3$ AND, $1/3$ OR, and $1/3$ XOR.
We compare the interaction proportions estimated by DMIL against the CVX estimator~\cite{liang2023quantifying}, a specialized method for quantifying interaction quantity of Redundancy, Uniqueness, and Synergy. For DMIL, we derive these proportions from its decomposed feature components. The results in \autoref{bool_types} show that although DMIL is not explicitly designed for quantification, it implicitly learns the interaction composition.
Our framework's estimates are highly comparable to those of the CVX estimator, validating the effectiveness of our decomposition approach. Furthermore, in complex datasets mixing two or three interaction types, DMIL often achieves a closer approximation to the ground truth, demonstrating its robustness in capturing multifaceted interactions.

\begin{table}[t!]
\centering
\caption{Validation of interaction decomposition on synthetic Boolean datasets. The table compares the estimated proportions (\%) of Redundancy ($\tilde R$), Uniqueness ($\tilde U^{(1)}$,$\tilde U^{(2)}$), and Synergy ($\tilde S$) interaction for various logical combinations with the ground truth (Truth).}
\resizebox{\linewidth}{!}{
\begin{tabular}{l|cccc|cccc|cccc|cccc}
\toprule
\multirow{2}{*}{Method} & \multicolumn{4}{c|}{AND} & \multicolumn{4}{c|}{AND+XOR} & \multicolumn{4}{c|}{OR+XOR} & \multicolumn{4}{c}{AND+OR+XOR} \\
\cmidrule(lr){2-5} \cmidrule(lr){6-9} \cmidrule(lr){10-13} \cmidrule(lr){14-17}
 & $\tilde{R}$ & $\tilde{U}^{(1)}$ & $\tilde{U}^{(2)}$ & $\tilde{S}$ & $\tilde{R}$ & $\tilde{U}^{(1)}$ & $\tilde{U}^{(2)}$ & $\tilde{S}$ & $\tilde{R}$ & $\tilde{U}^{(1)}$ & $\tilde{U}^{(2)}$ & $\tilde{S}$ & $\tilde{R}$ & $\tilde{U}^{(1)}$ & $\tilde{U}^{(2)}$ & $\tilde{S}$ \\
\midrule
CVX & 38.0 & 0.0 & 0.9 & 61.1 & 21.4 & 0.0 & 0.3 & 78.2 & 21.2 & 0.0 & 0.6 & 78.3 & 33.7 & 0.4 & 0.1 & 65.8 \\
DMIL & 36.4 & 0.0 & 1.7 & 61.9 & 19.0 & 0.4 & 0.0 & 80.6 & 20.7 & 0.0 & 1.9 & 77.4 & 26.8 & 4.1 & 0.0 & 69.1 \\
Truth & 38.3 & 0.0 & 0.0 & 61.7 & 19.1 & 0.0 & 0.0 & 80.9 & 19.1 & 0.0 & 0.0 & 80.9 & 25.5 & 0.0 & 0.0 & 74.5 \\
\bottomrule
\end{tabular}
}
\label{bool_types}
\end{table}

\subsection{Synthetic dataset}
\label{description_for_synthetic}
In Section 4.3, we construct two types of synthetic data to facilitate our study. The first type is shown in Figure 1 of the main paper, where the data is crafted with pre-defined interactions to elucidate specific interaction dynamics. The second type derives from Boolean logic variables. Here, the interactions are inherently embedded within the Boolean logic itself, providing a general consideration for interaction analysis.

The data generation process for predefined interactions is executed in two sequential steps. Initially, the type of interaction for each sample is identified. We categorize potential interactions as \textit{R}edundancy, \textit{U}niqueness, and \textit{S}ynergy for each dataset. As illustrated in the teaser figure in the introduction, each dataset is composed of samples exhibiting one or two types of interactions. The proportion of each interaction type is quantified using a fractional notation, such as $\frac{1}{4}U + \frac{3}{4}R$. This indicates that $\frac{1}{4}$ of the samples display \textbf{Unique} interactions, while the remaining samples demonstrate \textbf{Redundant} interactions.

In the second step, data corresponding to the predefined interactions is constructed. Different networks are employed to encode specific interactions into some dimensions of input space, which are then concatenated to form a comprehensive sample. If a sample is defined as a certain interaction, other types of interaction are suppressed by introducing Gaussian noise into their respective dimensions. This approach ensures that each sample exclusively embodies one type of interaction, thereby facilitating the construction of datasets with specified interaction properties.

The dataset, derived from Boolean logical data, is generated in a structured manner. Initially, the specific Boolean logic within each sample is determined. We consider two to three types of Boolean logics—\textit{AND}, \textit{OR}, and \textit{XOR}—with each sample containing only one type. Each type of logic occupies an equivalent proportion within the dataset. Subsequently, these logics are encoded into the input space. Given that Boolean data inherently contains measurable interactions~\cite{bertschinger2014quantifying}, we utilize this data to validate our method for interaction decomposition.

\end{document}